\documentclass[twoside,11pt]{article}
\usepackage{jmlr2e}
\usepackage{enumitem}
\usepackage{array}
\usepackage{multirow}
\usepackage{amsmath} 
\usepackage{amssymb}  
\usepackage[justification=centering]{caption}
\usepackage{algorithm,algorithmic}
\usepackage{pgfgantt} 
\usepackage{rotating}
\usepackage{fancyhdr}
\usepackage{thmtools,thm-restate}
\usepackage{etoolbox}
\newbool{insubexample}
\newtheorem{subexample}{Example}[example]
\AtBeginEnvironment{example}{\global\booltrue{insubexample}}%
\AtEndEnvironment{example}{\global\boolfalse{insubexample}}
\AtBeginEnvironment{subexample}{%
	\ifbool{insubexample}{}{\global\booltrue{insubexample}%
		\refstepcounter{example}}%
}%

\newcommand{\trace}{\operatorname{trace}}

\newcommand{\Span}{\operatorname{span}}
\newcommand{\Close}{\operatorname{closure}}

\jmlrheading{1}{2018}{1-48}{4/00}{10/00}{}{}

\ShortHeadings{Generalized Representer Theorems}{Diwale and Jones}

\begin{document}

\title{A Generalized Representer Theorem \\for Hilbert Space - Valued Functions}

\author{\name Sanket Diwale \email sanket.diwale@epfl.ch
       \AND
       \name Colin N.\ Jones \email colin.jones@epfl.ch\\
       \addr Automatic Control Laboratory\\
       {\'E}cole Polytechnique F{\'e}d{\'e}rale de Lausanne\\
       Lausanne, Switzerland}

\editor{}

\maketitle

\begin{abstract}
The necessary and sufficient conditions for existence of a generalized representer theorem are presented for learning Hilbert space - valued functions.
Representer theorems involving explicit basis functions and Reproducing Kernels are a common occurrence in various machine learning algorithms like generalized least squares, support vector machines, Gaussian process regression and kernel based deep neural networks to name a few. Due to the more general structure of the underlying variational problems, the theory is also relevant to other application areas like optimal control, signal processing and decision making. We present the generalized representer as a unified view for supervised and semi-supervised learning methods, using the theory of linear operators and subspace valued maps. The implications of the theorem are presented with examples of multi input - multi output regression, kernel based deep neural networks, stochastic regression and sparsity learning problems as being special cases in this unified view.
\end{abstract}

\begin{keywords}
  Linear Operators, Adjoints, Kernels, Representer Theorems
\end{keywords}

\section{Introduction}

The development of kernel based methods for regression and machine learning has a long history with several algorithms basing themselves on the Reproducing Kernel Hilbert Space (RKHS) theory. Some of the early works in the field include \cite{Aronszajn1950,Tikhonov:1963,Wahba1990}, which looked at problems of spline interpolation and smoothing in the RKHS setting. Several practical learning algorithms like linear regression, support vector machines, Bayesian regression were also developed in their kernel forms to allow more complex nonlinear representations of data \citep[see][for some examples]{Bishop:2006:PRM:1162264}. Kernel based stochastic models are also popular in the form of Gaussian Process models \citep[see][]{Rasmussen06}. RKHS based neural networks have been investigated in \cite{Lawrence09,Rebai2016,DamianouL13}. 

Representer theorems provide a means to reduce infinite dimensional optimization problems for learning in the RKHS space to an equivalent and tractable finite dimensional optimization.
While works like \cite{Micchelli:2005:LVF:1119610.1119620,Minh:2011:VMR:3104482.3104490,JMLR:v17:14-036} present representer theorems for RKHS based learning methods for vector valued functions in Hilbert spaces, the theorems are presented independently for each specific learning algorithm. A Generalized Representer Theorem covering an entire class of learning algorithms for infinite dimensional vector valued outputs, to the best of the authors knowledge, is still missing from literature. The broadest form, so far, of such a Generalized Representer Theorem is presented in \cite{Argyriou:2014:UVR:3044805.3044976} which covers learning problems with finite dimensional vector valued outputs and requires a technical assumption of ``r-regularity" on the subspace valued maps used in the theorem, requiring a finite dimensional span for the subspace valued maps. 
We extend the framework of the Generalized Representer Theorem to arbitrary infinite dimensional vector output spaces and remove the ``r-regularity" assumption. As a consequence we cover stochastic process learning algorithms that were not covered by previous counterparts of the generalized theorem and allow for other more general infinite dimensional output learning algorithms.
We also show the $\ell_1$ regularization problems to be covered by the Generalized Representer Theorem using a non trivial, non $r$-regular  subspace valued map. 

A key underlying tool in the use of RKHS methods is the Riesz Representer Theorem \cite[Theorem 3.3.1]{conway_course_abstract} and the existence and uniqueness of adjoint operators for bounded linear operators given by \cite[Theorem 5.4.2]{conway_course_abstract}. The above two theorems combined with restrictions on the forms of the objective and constraint functionals in learning problems have led to several variants of Representer Theorems. 
Early variants of representer theorems are presented in \cite{Wahba1990} for variational problems in learning real valued functions with least squares regularization. Representer theorems for kernel versions of different learning algorithms like Kernel SVM, PCA, CCA, ICA can be found in \cite{suykens2010}. This has prompted investigation into unifying  representer theorems into a single generalized theorem and characterizing the class of problems for which a representer theorem can be guaranteed to exist. 

The first such results appear to have come from \cite{Scholkopf:2001:GRT:648300.755324}, where the problem has been addressed for learning real valued functions with a general class of regularizers and empirical risk functions. The regularizers considered were a class of monotonically increasing functions of the norm of decision variables and showed how most of the least squares algorithms in linear regression, SVMs and others were covered by a single generalized theorem. The work provides a sufficient condition for the existence of such representer theorems. \cite{DinuzzoS2012} relaxed the restriction on the regularizer further and provided necessary and sufficient conditions for the existence of representer theorems. \cite{DinuzzoS2012} allow the regularizer to be any lower semi-continuous functional on the decision variable as long as the functional satisfies an ``Orthomonotone" property. 
\cite{Scholkopf:2001:GRT:648300.755324,DinuzzoS2012} restricted the scope of their theorem to learning real valued functions. The generalized theorem was extended to learning multi-output functions in \cite{Argyriou:2014:UVR:3044805.3044976} for finite dimensional outputs. We extend this work here further to arbitrary Hilbert space outputs and remove a finite dimensional ``r-regularity" assumption made in \cite{Argyriou:2014:UVR:3044805.3044976}. While finite dimensional multi-output learning algorithms cover a relatively large class of algorithms it still leaves out the more general cases of stochastic or Bayesian regression and more general cases of learning mappings between abstract vector spaces. For example representer theorems for Bayesian regression from \cite{pillai2007characterizing} are not covered by previous works as noted in \cite{Argyriou:2014:UVR:3044805.3044976}. The extension to Hilbert space valued outputs allows us to treat these more general cases within the framework of a generalized representer theorem.
While further generalization beyond Hilbert spaces to Banach spaces may be possible using the notions of dual vector spaces, we will restrict ourselves to Hilbert spaces here, to maintain clarity in exposition. 

The main contributions made in this paper are: (i) extend the framework of Generalized Representer Theorems to learning infinite dimensional Hilbert space valued functions, (ii) remove the restriction of ``r-regularity" assumption from \cite{Argyriou:2014:UVR:3044805.3044976}, (iii) present examples of learning problems from stochastic and deterministic settings where the extensions presented have implications.

In Section \ref{sec:prelim} we present the preliminaries required to develop the generalized representer theorem for Hilbert space valued functions. Section \ref{sec:LinearOperators} presents some background material on linear operators and their adjoints which play a key role in developing kernel representations for Hilbert spaces. Section \ref{sec:subspaceValMaps} presents the notion of a subspace valued map that plays a key role in the proof of the generalized representer theorem. We introduce here the notion of super additive subspace valued maps that replaces the notions of quasilinear, idempotent maps used in prior counterparts of the theorem and provides a necessary and sufficient condition for subspace valued maps to preserve the structure of closed vector subspaces. Section \ref{sec:Orthomonotonicity} presents the notion of Orthomonotone functionals with respect to a subspace valued map. The $\ell_1$ regularizer is shown here to be orthomonotone with respect to a non trivial subspace valued map that enables the application of the Generalized theorem to  $\ell_1$ regularizing problems. The Generalized Representer Theorem giving necessary and sufficient conditions for the existence of a representer for learning Hilbert space-valued functions is then presented in 
Section \ref{sec:GRT}. Section \ref{sec:examples} presents examples of learning algorithms covered by this extension. The appendix provides proofs for lemmas used the paper and provides definitions for quasilinear, idempotent and $r$-regular subspace valued maps for reference.
\section{Preliminaries}
\label{sec:prelim} 
The notions of continuous linear operators, adjoint linear operators, subspace valued maps and orthomonotone functionals are introduced.
\subsection{Continuous Linear Operators}
\label{sec:LinearOperators}
Let $\mathcal{L}_{\mathcal{H},\mathcal{Z}}$ be the space of continuous linear operators $L:\mathcal{H}\to\mathcal{Z}$ for arbitrary Hilbert spaces $\mathcal{H}$ and $\mathcal{Z}$.
Let $\langle \cdot,\cdot \rangle_{\mathcal{H}}$, $\langle \cdot,\cdot \rangle_{\mathcal{Z}}$ be the inner products defined on $\mathcal{H}$ and $\mathcal{Z}$ respectively.
For any fixed $L\in\mathcal{L}_{\mathcal{H},\mathcal{Z}}$, by the Riesz representation theorem on Hilbert spaces, there exists a unique continuous linear operator $L^*:\mathcal{Z}\to\mathcal{H}$, called the adjoint to $L$, such that 
\begin{equation}
\langle z, Lf \rangle_{\mathcal{Z}} = \langle f,L^* z  \rangle_{\mathcal{H}} \qquad \forall z\in \mathcal{Z}, \forall f \in \mathcal{H}
\end{equation}
\citep[see][Theorem 5.4.2, for a formal proof]{conway_course_abstract}

To characterize the null space of a linear operator $L\in\mathcal{L_{\mathcal{H},\mathcal{Z}}}$ consider the following lemma,
\begin{lemma}
	\label{lem:orthonullspace}
	Let $\mathcal{N}_L$ be the null space of $L\in\mathcal{L_{\mathcal{H},\mathcal{Z}}}$ and  $\mathcal{N}_L^\perp$ be its orthogonal complementary space, then,
	$$\mathcal{N}_L^\perp=\{L^*z:z\in\mathcal{Z}\}$$
	\begin{proof}
		Let $P_L:=\{L^*z:z\in\mathcal{Z}\}$ and note that $P_L$ is a closed subspace of $\mathcal{H}$.
		To characterize the null space of $L$ observe that $Lg = 0$ if and only if, $\forall z\in\mathcal{Z}$,
		$\langle Lg,z\rangle_\mathcal{Z} = \langle g,L^*z\rangle_\mathcal{H} = 0$. Thus the null space is characterized by $\mathcal{N}_L:=\{g\in\mathcal{H}:\langle g,L^*z\rangle_\mathcal{H} = 0, \quad \forall z\in\mathcal{Z}\} = P_L^\perp$. Then $\mathcal{N}_L^\perp=P_L=\{L^*z:z\in\mathcal{Z}\}$. 
	\end{proof}
\end{lemma}
Thus the adjoint operator plays a key role in characterizing the null space of an operator $\mathcal{N}_L$ and its orthogonal complementary space $\mathcal{N}_L^\perp$.
\begin{corollary}
	\label{cor:nullspaceorth}
	Given a set of operators $L_1,\dots,L_m$, the joint null space is $\mathcal{N}_{L_1,\dots,L_m} = \mathcal{N}_{L_1}\cap\dots\cap\mathcal{N}_{L_m}$ and $\mathcal{N}_{L_1,\dots,L_m}^\perp = \Close(\mathcal{N}_{L_1}^\perp+\dots+ \mathcal{N}_{L_m}^\perp) = \Close(\{\sum_{i=1}^{m}L_i^*z_i:z_i\in\mathcal{Z}\})$.
\end{corollary}

\subsubsection{Adjoint for operators of common interest}
Below we show a few examples of adjoint operator for operators commonly seen in learning algorithms.
\begin{example}{Evaluation Operators\\}
	\label{ex:EvalOp}
	Let $\mathcal{H}$ be a space of functions $f:\mathcal{X}\to\mathcal{Z}$.
	Then a parametric linear evaluation operator $L_{x}:\mathcal{H}\to\mathcal{Z}$ is given by $L_{x}(f) = f(x)$ for some fixed parameter $x\in\mathcal{X}$. This operator commonly occurs in machine learning and data fitting problems where $x$ is the training input data and $f(x)$ gives a predicted value for the output in $\mathcal{Z}$.
	The adjoint $L_{x}^*:\mathcal{Z}\to\mathcal{H}$ can be found as follows.
	
	Note that by definition of $L_x$ and its adjoint $L_{x}^*$, $\forall g\in\mathcal{H},z\in\mathcal{Z}$,  $\langle L_{x}^*z,g \rangle_\mathcal{H} = \langle L_{x}g,z\rangle_{\mathcal{Z}}$, i.e., $\langle L_{x}^*z,g \rangle_\mathcal{H} = \langle g(x),z\rangle_\mathcal{Z}$.
	When $\mathcal{H}$ is a reproducing kernel Hilbert space with kernel $K$, $L_{x}^*$ is well defined and coincides with the definition of the RKHS kernel \citep[see][Definition 2.1]{Micchelli:2005:LVF:1119610.1119620}. Thus RKHS spaces provide a case where the adjoint operator for evaluation operators is well defined and $L_x^*=K(\cdot,x)$. 
\end{example}

\begin{example}{Linear Transformations of explicit basis\\}
	Given a fixed vector valued function $\phi:\mathcal{X}\to\mathcal{Y}$. Let $\mathcal{H},\mathcal{Z}$ be arbitrary Hilbert spaces and let $\ell:\mathcal{H}\to \mathcal{L}_{\mathcal{Y},\mathcal{Z}}$ be a linear map from $\mathcal{H}$ to continuous linear operators mapping $\mathcal{Y}$ to $\mathcal{Z}$. Then we can define a continuous linear operator $L_{x,\phi}:\mathcal{H}\to\mathcal{Z}$ for any $W\in\mathcal{H}$ as $L_{x,\phi}(W):=\ell(W)\phi(x)$. The adjoint operator must then satisfy $\langle L_{x,\phi}^*z,W\rangle_\mathcal{H}=\langle \ell(W)\phi(x),z  \rangle_\mathcal{Z}$.
	\label{Ex:LinearTransform}
	
	\begin{subexample}{Finite dimensional example\\}
		\label{ex:finiteexpbasis}
			Let $\mathcal{Y}=\mathbb{R}^n$, $\mathcal{Z}=\mathbb{R}^k$ and $\mathcal{H}=\mathbb{R}^{n\times k}$. Then $\phi(x)\in\mathbb{R}^n$ and let $\ell(W):=W^T$ which is an operator from $\mathcal{Y}$ to $\mathcal{Z}$. Then for any $W\in\mathcal{H}$, $L_{x,\phi}(W)= W^T\phi(x)$.
		
		Let the inner product on $\mathcal{H}$ be the Forbenius inner product of matrices, i.e, $\langle w_1,w_2  \rangle_{\mathcal{H}}=\trace(w_1^Tw_2)$. Let inner product on $\mathcal{Z}$ be $\langle z_1,z_2\rangle_\mathcal{Z}=z_1^Tz_2$. Then for the adjoint operator $\langle L_{x,\phi}^*z,W\rangle_\mathcal{H}=\langle W^T\phi(x),z  \rangle_\mathcal{Z}$, $\forall z\in \mathcal{Z}$ implying $\trace(W^T L_{x,\phi}^*z) = \phi(x)^TWz$. Noting then that $\phi(x)^TWz = \trace(\phi(x)^TWz)=\trace(z^TW^T\phi(x))=\trace(W^T\phi(x)z^T)$, we can define $L_{x,\phi}^*z:=\phi(x)z^T$. Further we know that this must be the unique adjoint operator for the defined inner products \citep[since uniqueness of the adjoint is guaranteed by][Theorem 5.4.2]{conway_course_abstract}.
	\end{subexample}
	
	\begin{subexample}{Infinite dimensional example\\}
		Let $\mathcal{X}=\mathbb{R}^{n}$, $\mathcal{U}=\mathbb{R}^m$ and $\mathcal{H}=\mathbb{R}^{m\times N}$. Let $\{\mathcal{Y}_i: i = 1,\dots,N\}$ be a collection of RKHS spaces of functions $f:\mathcal{X}\to\mathcal{U}$ with kernels $K_1,\dots,K_N$. Let $\mathcal{Y}=\mathcal{Y}_1 \times \dots \times \mathcal{Y}_N$ and $\phi(x)=\begin{pmatrix}
		K_1(\cdot,x) \\ K_2(\cdot,x) \\ \vdots \\ K_N(\cdot,x)
		\end{pmatrix}$ and let $\mathcal{Z}$ be some infinite dimensional Hilbert space of functions $g:\mathcal{X}\to\mathcal{U}$ with inner product $\langle g_1,g_2 \rangle_{\mathcal{Z}} = \int_{\mathcal{X}}\langle g_1(x),g_2(x)\rangle_\mathcal{U}dx$. Then we can define a continuous linear operator $L_{x,\phi}:\mathcal{H}\to\mathcal{Z}$ for any $W\in\mathcal{H}$ as $L_{x,\phi}(W):=\sum_{i=1}^{N}K_i(\cdot,x)W_i$, with $W_i$ denoting the $i^{th}$ column of $W$. Using the Forbenius inner product on $\mathcal{H}$, $\langle L_{x,\phi}^*g , W\rangle_{\mathcal{H}}=\langle L_{x,\phi}(W),g \rangle_{\mathcal{Z}}=\sum_{i=1}^{N}\int_{\mathcal{X}}\langle K_i(y,x)W_i,g(y) \rangle_{\mathcal{U}} dy = \sum_{i=1}^{N}\int_{\mathcal{X}} W_i^TK_i(x,y)g(y)dy$. Also note that $\langle L_{x,\phi}^*g , W\rangle_{\mathcal{H}}= \trace(W^T L_{x,\phi}^*g) = \sum_{i=1}^{N} W_i^T[L_{x,\phi}^*g]_i$. Thus $[L_{x,\phi}^*g]_i = \int_{\mathcal{X}} K_i(x,y)g(y)dy$ gives the adjoint.
	\end{subexample}
\end{example}


\begin{example}{Derivative Operator\\}
	\label{ex:derOp}
	Let $\mathcal{H}$ be the space of once differentiable functions $f:\mathcal{X}\to \mathcal{Z}$ with compact support. Let $\mathcal{Z}=\mathbb{R}^m$ and $\mathcal{X}=\mathbb{R}^n$. Let $\mathcal{Y}$ be a space of $\mathbb{R}^{m\times n}$ valued functions and let $\mathcal{W}=\mathbb{R}^{m\times n}$. Let $D:\mathcal{H}\to\mathcal{Y}: D(f)
	 := (\partial_{x_1} f, \dots, \partial_{x_n}f)$ be a derivative operator and 
	let $D_x:\mathcal{H}\to \mathcal{W}$ be the evaluation of the differential at some $x\in\mathcal{X}$. Let the inner product on $\mathcal{W}$ be given by the Forbenius matrix inner product and let the inner product on $\mathcal{Y}$ be given by $\langle f,g \rangle_{\mathcal{Y}}=\int_{\mathcal{X}}\langle f(x),g(x) \rangle_{\mathcal{W}}dx$. 
	Then for any $w\in\mathcal{W}$, \[\left\langle {Df}(x), w \right\rangle_{\mathcal{W}} := \trace\left({Df}(x)^T w\right) =  \sum_{i=1}^{n} \sum_{j=1}^{m} \left[{Df}(x)\right]_{ji}[w]_{ji} \]
	Integrating by parts,
	\begin{eqnarray*}
		\left\langle {Df}, g \right\rangle_{\mathcal{Y}} &=& \sum_{j=1}^{m} \sum_{i=1}^{n} \int_{\mathcal{X}} [g(t)]_{ji} \left[{Df}(t)\right]_{ji}dt\\
		 &=& \sum_{i=1}^{n} \sum_{j=1}^{m} \left[ [f(t)]_{j}[g(t)]_{ji}\bigg|_\mathcal{\partial X} - \int_{\mathcal{X}}[f(t)]_{j}\partial_{x_i}\left[g\right]_{ji}(t)dt \right]\\
		&=&  - \sum_{i=1}^{n} \sum_{j=1}^{m} \int_{\mathcal{X}}[f(t)]_j\partial_{x_i}\left[g\right]_{ji}(t)dt
	\end{eqnarray*}
	Assuming compact support for $f$, the boundary terms on $\partial\mathcal{X}$ go to zero. Comparing to the inner product $\langle D^*g, f\rangle_\mathcal{H} :=\int_{\mathcal{X}} f(t)^T(D^*g)(t) dt = \int_{\mathcal{X}}\sum_{j=1}^{m} [f(t)]_j [D^*g]_{j}(t)dt$,
	\[ [D^*g]_j = - \sum_{i=1}^{n} \partial_{x_i} [g]_{ji} \quad \textrm{for } j=\{1,\dots,m\}\]
\end{example}

\subsection{Subspace Valued Maps}
\label{sec:subspaceValMaps}
The notion of subspace valued maps was introduced in \cite{Argyriou:2014:UVR:3044805.3044976} for the proof of their generalized representer theorem. We introduce below the same notion and present further properties of such functions when composed with linear operators. 
\begin{remark}{(Extension to previous works)}
	The subspace valued maps used in \cite{Argyriou:2014:UVR:3044805.3044976} were required to have a finite rank property called $r$-regularity. With our general formulation for infinite dimensional outputs this property is no longer required. The notions of quasilinear and idempotent subspace valued maps were also used in the previous work and we provide an alternative characterization with ``super additivity" here to better suit the requirements of the problem at hand. The differences between idempotent, quasilinear maps and super additive maps are explained in further remarks below. The terms,  Inclusive and Closed subspace valued map are also introduced here and it is noted that all subspace valued maps considered in the previous work were Inclusive and Closed without using these terms explicitly. The notions of quasilinear, idempotent and $r$-regular subspace valued maps from \cite{Argyriou:2014:UVR:3044805.3044976} are defined in the Appendix. 
\end{remark}	
\begin{remark}{(Notation)\\}
	Let $U$ be a vector space on some field $\mathcal{K}$ with an addition operation $+_U$ and a scalar multiplication $\circ_U$.
	Let $A, B$ be two subsets of $U$. Then for any $\lambda\in\mathcal{K}$, we denote by $\lambda \cdot A$ a new set $A':=\{\lambda \circ_U a:a\in A\}$. Similarly, $A+B:=\{a+_U b:a\in A, b \in B\}$. Let $\mathbb{F}=2^U$ be the power set of $U$.  We denote $\Span(\mathbb{F}):=\{\cup_{\lambda\in\mathcal{K}}\lambda A:A\in\mathbb{F}\}$.
\end{remark} 

\begin{definition}{(Subspace valued map)\\}
	For any given set of sets $\mathbb{F}$,
	a map $S:\mathbb{F}\to\Span(\mathbb{F})$ is called \textbf{subspace valued}.
\end{definition}

\begin{definition}{(Inclusive map)\\}
	We call a subspace valued map $S:\mathbb{F}\to\Span(\mathbb{F})$ \textbf{inclusive} if for all vector subspaces $A\subseteq U$, $A\subseteq S(A)$
\end{definition}
\begin{definition}{(Closed map)\\}
	We call a subspace valued map $S:\mathbb{F}\to\Span(\mathbb{F})$ \textbf{closed} if, for all closed sets $A\in\mathbb{F}$, $S(A)$ is a closed set, i.e., for all convergent sequences (nets) $\{v_n\in S(A)\}$ converging to $v$ in norm, $v\in S(A)$.
\end{definition}
\begin{definition}{(Super additive map)\\}
	A map $S:\mathbb{F}\to\Span(\mathbb{F})$ is called \textbf{super additive} if for all vector subspaces $A,B\subseteq U$, \[S(A)+S(B)\subseteq S(A+B)\]
\end{definition}
\begin{definition}{(Orthogonal subspace)\\}
	Let $(U,+_U,\circ_U)$ be associated with an inner product $\langle \cdot, \cdot \rangle_U$, then 
	for any $A\in\mathbb{F}$, we define $S(A)^\perp:=\{b\in U: \forall a\in S(A),\langle a,b\rangle_U=0\}$
\end{definition}
\begin{remark}{(Extending maps on sets to maps on members)\\}
	A subspace valued map $S:\mathbb{F}\to\Span(\mathbb{F})$, can be extended for evaluation for any $u\in U$ by interpreting $S(u)$ as $S(\{u\})$. This is useful for shortening notation when talking simultaneously of evaluating $S$ on sets as well as individual members of $U$.
	\label{rem:InheritedF}
\end{remark}
\begin{example}{{Subspace valued maps}}
	\label{ex:subvalmaps}
	\begin{enumerate}
		\item $S_{\mathbb{R}}(A):=\{\lambda a: a\in A,\lambda\in\mathbb{R} \}$ is an inclusive, closed, super additive subspace valued map 
		\item Consider $U=\mathbb{R}^2$ and let $\theta$ be a fixed angle in $(0,\pi)$ radians. Denote by $R_\theta:U\to U$ a rotation transform on a vector in $\mathbb{R}^2$ that rotates the vector by $\theta$ radians clockwise. Then $S_\theta(A):=\{\lambda R_{\theta}a: a\in A, \lambda\in\mathbb{R} \}$ is a closed, super additive subspace valued map. However $S_{\theta}$ is not inclusive. 
		\item Consider $U$, $\theta$ and rotation operators $R_{(\cdot)}$ from the previous example. Let $S_\phi(A):=\{\lambda R_{\phi}a: a\in A, \lambda\in\mathbb{R},\phi\in(-\theta,\theta) \}$ and $S_\psi(A):=\{\lambda R_{\psi}a: a\in A, \lambda\in\mathbb{R},\psi\in[-\theta,\theta] \}$.  Both $S_\phi$ and $S_\psi$ are inclusive and super additive. $S_\psi$ is closed, but $S_\phi$ is not.
		\item Consider $U$, $\theta$ and rotation operators $R_{(\cdot)}$ from the previous example and $S_{\mathbb{R}}$ subspace valued map from the first example. Let $S_{\pi/2}(A):=\{ \lambda R_{\pi/2}a: a\in A, \lambda\in\mathbb{R} \} \cup S_{\mathbb{R}}(A)$. $S_{\pi/2}$ is inclusive and closed but not super additive. 
		\item Let $\mathcal{L}_{U,U}$ be a closed vector space of continuous linear operators $L:U\to U$. Then $S_\mathcal{L}(A):=\{La: a\in A, L\in\mathcal{L}_{U,U}\}$ is closed, inclusive and super additive. 
		\item Let $U=\mathbb{R}^n$ and $E=\{e_1,\dots,e_n\}$ be the standard orthonormal basis for $\mathbb{R}^n$. Then $S_{null}(A):=\{\lambda\langle a,e_i\rangle_Ue_i:e_i\in E,a\in A,\lambda\in\mathbb{R}\}$ is a closed subspace valued map. $S_{null}$ is not inclusive or super additive. 
		\item Let $U=\mathbb{R}^n$ and $E=\{e_1,\dots,e_n\}$ be the standard orthonormal basis for $\mathbb{R}^n$. Then $S_{proj}(A):=\{\sum_{i=1}^{n}\lambda_i\langle a,e_i\rangle_Ue_i:e_i\in E,a\in A,\lambda_i\in\mathbb{R}\}$ is an inclusive, closed, super additive subspace valued map. 
	\end{enumerate}
\end{example}
Note that if $A\subseteq U$ is a vector subspace of $U$, $S(A)$ need not be a vector space of $U$ as well. For example $S_\theta, S_{\pi/2}, S_{null}$ are all valid subspace valued maps but do not always map a vector subspace $A$ to another vector subspace. Further $S_{\pi/2}$ is an example of an inclusive, closed, quasilinear, idempotent map for which $A$ being a vector space does not imply $S(A)$ to be a vector space, showing that quasilinearity and idempotence are not sufficient to preserve a vector space structure. This was also noted in \cite[Remark 2.1]{Argyriou:2014:UVR:3044805.3044976}.
To ensure that $S(A)$ remains a closed vector subspace of $U$ it is necessary and sufficient for $S$ to be closed and super additive as shown by the lemma below.

\begin{lemma}
	For all closed vector subspace $A\subseteq U$, $S(A)$ is a closed, vector subspace of $U$ if and only if $S$ is a closed, super additive subspace valued map.
	\begin{proof}
		Note that for any $a,b\in S(A)$ there exists vector subspaces $v_1,v_2\subseteq A$ such that $a\in S(v_1)$ and $b\in S(v_2)$. Then for any $\lambda_1,\lambda_2\in\mathbb{R}$, $\lambda_1 a + \lambda_2 b \in S(v_1)+S(v_2)$. If $S$ is super additive then $S(v_1)+S(v_2)\subseteq S(v_1+v_2)$. Also since $v_1,v_2\subseteq A$ are subspaces in $A$, $v_1+v_2\subseteq A$, implying $S(v_1+v_2)\subseteq S(A)$. Thus if $S$ is super additive, for any $\lambda_1,\lambda_2\in\mathbb{R}$, $a,b\in S(A)$, $\lambda_1 a + \lambda_2 b \in S(A)$. Thus $S(A)$ is a vector space. Further for $S(A)$ to be closed, $S$ must be a closed.
		
		To show necessity of super additive $S$, we proceed by contradiction. Let $S$ not be super additive but $S(A)$ be vector subspace for all vector subspaces $A$. Then there exist a vector subspace $A\subseteq U$ and subspaces $v_1,v_2\subseteq A$ such that $S(v_1)+S(v_2) \nsubseteq S(v_1+v_2)$. But $v_1+v_2$ is a vector subspace of $A$ and $v_1,v_2\subseteq v_1+v_2$. Thus $S(v_1)\subseteq S(v_1+v_2)$ and $S(v_2)\subseteq S(v_1+v_2)$. Also since $v_1+v_2$ is a vector space and $S(v_1+v_2)$ is a vector space by assumption, then $S(v_1)\subseteq S(v_1+v_2)$, $S(v_2)\subseteq S(v_1+v_2)$ implies $S(v_1)+S(v_2)\subseteq S(v_1+v_2)$, which contradicts the assumption of $S$ not being super additive. Thus $S(A)$ is a closed vector space for all closed, vector space $A$ if and only if $S$ is super additive.
	\end{proof}
\end{lemma}

The notions of quasilinear and idempotent maps from prior work are related to the notion of super additivity by noting that for any quasilinear, idempotent $S$, $S_{sup}(A):=\sum_{w\in A}S(w)$ can be defined as the corresponding super additive map.

Another property that is of interest for us is the preservation of the null space for a collection of operators under a subspace valued map. Formally we define this property as follows,
\begin{definition}{(Null space preserving map)\\}
	Let $L_1,\dots,L_m$ be continuous linear operators and let $A = \mathcal{N}_{L_1,\dots,L_m}^\perp$ be the orthogonal subspace to the joint null space of the operators. Then a subspace valued map $S:\mathbb{F}\to\Span(\mathbb{F})$ is called \textbf{Null space preserving} with respect to operators $\{L_1,\dots,L_m\}$ if
	\[ S(A)^\perp \subseteq \mathcal{N}_{L_1,\dots,L_m} \]
\end{definition}

When $S$ is null space preserving with respect to $\{L_1,\dots,L_m\}$, then for all $g\in S(A)^\perp$, $L_ig = 0$ for all $i\in\{1,\dots,m\}$. This fact will be useful later when proving the generalized theorem.

We note that all inclusive maps are null space preserving but not vice versa. However a closed, super additive $S$ is null space preserving if and only if $S$ is inclusive. Lemma \ref{lem:incimpliesnullpres} below shows inclusive maps to be null space preserving. Lemma \ref{lem:Sprojnullspace} shows a null space preserving map that is not inclusive. Finally Lemma \ref{lem:closedQuasInc} shows that inclusivity is necessary and sufficient for $S$ to be null space preserving if $S$ is closed and super additive.
\begin{lemma}{(Inclusive implies null space preserving)\\}
	\label{lem:incimpliesnullpres}
	If $S:\mathbb{F}\to\Span(\mathbb{F})$ is inclusive then it is null space preserving.
	\begin{proof}
		Let $A = \mathcal{N}_{L_1,\dots,L_m}^\perp$. Then note that $S$ being inclusive, implies $A\subseteq S(A)$. Also for all $g\in S(A)^\perp$ and $f\in A$ ($\subseteq S(A)$), $\langle f,g\rangle_{\mathcal{H}}=0$, implying $g\in A^\perp=\mathcal{N}_{L_1,\dots,L_m}$, i.e., $S(A)^\perp \subseteq \mathcal{N}_{L_1,\dots,L_m}$.
	\end{proof}
\end{lemma}
\begin{lemma}{($S_{null}$ as null space preserving map)\\}
	\label{lem:Sprojnullspace}
	Let $S_{null}$ be the projected subspace value as defined in Example \ref{ex:subvalmaps}-4. $S_{null}$ is null space preserving.
	\begin{proof}
		Let $A = \mathcal{N}_{L_1,\dots,L_m}^\perp$ and $E=\{e_1,\dots,e_n\}$ be the standard basis for $\mathbb{R}^n$. Then $S_{null}(A)^{\perp}=\{\lambda e_j: \forall f\in A, \langle f,e_j\rangle_{\mathbb{R}^n}=0,\lambda\in\mathbb{R}\}$, i.e., $S_{null}(A)^{\perp}\subseteq A^\perp$. Thus $S_{null}(A)^{\perp}\subseteq \mathcal{N}_{L_1,\dots,L_m}$ implying $S_{null}$ is null space preserving.
	\end{proof}
\end{lemma}
Thus $S_{null}$ provides an example of a subspace valued map that is null space preserving but not inclusive. 
\begin{lemma}{(Closed, super additive and inclusive $S$)\\}
	\label{lem:closedQuasInc}
	Let $S$ be a closed, super additive subspace valued map.
$S$ is null space preserving with respect to operators $\{L_1,\dots,L_m\}$ 
if and only if $S$ is inclusive.
\begin{proof}
	Let $A=\mathcal{N}_{L_1,\dots,L_m}^\perp$.
	If $S$ is inclusive then it is null space preserving, by Lemma \ref{lem:incimpliesnullpres}. On the other hand if $S$ is null space preserving, then $S(A)^\perp \subseteq \mathcal{N}_{L_1,\dots,L_m}$, implying $\mathcal{N}_{L_1,\dots,L_m}^\perp \subseteq (S(A)^\perp)^\perp$. But $\mathcal{N}_{L_1,\dots,L_m}^\perp = A$ and $(S(A)^\perp)^\perp= S(A)$ ($\because$ $S(A)$ is a closed vector subspace and $S(A)$ and $S(A)^\perp$ are orthogonal complementary vector subspaces, by virtue of $S$ being closed and super additive). Thus $A\subseteq S(A)$, i.e. $S$ is inclusive. Thus a closed, super additive $S$ is null space preserving if and only if $S$ is inclusive.
\end{proof}
\end{lemma}
The null space preserving property and orthogonal complementary nature of $S(A)$ and $S(A)^\perp$ will be key in characterizing the  conditions for the existence of a representer theorem. Thus from here on we will only be interested in closed, super additive and inclusive subspace valued maps. We next establish these properties for subspace valued maps when composed with continuous linear operators.
\subsubsection{Composition with Linear Operators}
	Let $\mathcal{H},\mathcal{Z}$ be two Hilbert spaces. Let $L:\mathcal{H}\to\mathcal{Z}$ be a continuous linear operator. Let $\mathbb{F_\mathcal{Z}}$, $\mathbb{F}_\mathcal{H}$ be the power set of $\mathcal{Z}$,$\mathcal{H}$ respectively. Let $S:\mathbb{F}_\mathcal{Z}\to\Span(\mathbb{F}_\mathcal{Z})$ be an inclusive, closed super additive subspace valued map in $\mathcal{Z}$. We would like to define a new subspace valued map $S_L:\mathbb{F}_\mathcal{H}\to\Span(\mathbb{F}_\mathcal{H})$ in $\mathcal{H}$ that preserves the closed, inclusive and super additive properties of $S$. The following proposition (proof in Appendix) defines one such map.
	
\begin{restatable}{proposition}{SLprop}{(Pulling back subspace valued maps)\\}
	\label{prop:pullbackS}
	Let $S:\mathbb{F}_\mathcal{Z}\to\Span(\mathbb{F}_\mathcal{Z})$ be an inclusive, closed and super additive subspace valued map in $\mathcal{Z}$.
 Then, $S_L:\mathbb{F}_\mathcal{H}\to\Span(\mathbb{F}_\mathcal{H})$ defined as,
 $	S_L(A) := L^* S(L(A))$
is a closed and super additive, (not necessarily inclusive), subspace valued map in $\mathcal{H}$.
\end{restatable}

$S_L$ is guaranteed to be inclusive if $L$ is unitary (i.e. $L^*=L^{-1}$). For non-unitary $L$, inclusivity of $S_L$ cannot be guaranteed in general, however examples of inclusive $S_L$ can be produced for certain combinations of $\mathcal{H}$ and $S$ definitions. Lemma \ref{lem:Sproj_Dinclusive} in the appendix shows one such combination for the derivative operator $D$ from Example \ref{ex:derOp}.

Further to maintain orthomonotone properties, $L$ must preserve orthogonality for $S_L(A)$ and $S_L(A)^\perp$. If $L$ is unitary ($L^*=L^{-1}$) and $S$ is inclusive, $LS_L(A)=LL^*S(LA)=S(LA)$ and from Lemma \ref{lemma:OrthS} from the appendix, we know for any arbitrary $L$ and $S$, $L(S_L(A)^\perp)\subseteq S(LA)^\perp$. Thus $L(S_L(A))$ and $L(S_L(A)^\perp)$ are orthogonal subspaces when $L$ is unitary. 

For non-unitary $L$, $LS_L(A)\nsubseteq S(LA)$ in general and thus $L(S_L(A))$ and $L(S_L(A)^\perp)$ are not orthogonal subspaces in general. However for certain combinations of $\mathcal{H}$ and $S$ definitions, we can still have $LS_L(A)\subseteq S(LA)$ for non unitary $L$, thus maintaining the orthogonality of subspaces $L(S_L(A))$ and $L(S_L(A)^\perp)$. Lemma \ref{lem:DpresOrtho} in the appendix shows this to be the case for the derivative operator $D$ when $S_{proj}$ (from Example \ref{ex:subvalmaps}) is used as the subspace valued map on $\mathcal{Z}$.
\begin{definition}{(Preserving orthogonality w.r.t. $S_L$)\\}
	\label{def:orhtpresL}
	Let $S_L$ be an inclusive, closed and super additive. $L:\mathcal{H}\to\mathcal{Z}$ is said to preserve orthogonality with respect to $S_L$, if 
	\begin{enumerate}
		\item $LS_L(A) \subseteq S(LA)$
		\item $L(S_L(A)^\perp) \subseteq S(LA)^\perp$
	\end{enumerate}
\end{definition}

\subsection{Orthomonotone Functionals}
\label{sec:Orthomonotonicity}
\begin{definition}
	Let $\mathcal{Z}$ be a Hilbert space. A functional $\Omega:\mathcal{Z}\to\mathbb{R}\cup\{+\infty\}$ is called \textbf{Orthomonotone} with respect to a map $S:\mathbb{F}_\mathcal{Z}\to\Span(\mathbb{F}_\mathcal{Z})$ if \[\forall A\in\mathbb{F}_\mathcal{Z}, \forall f\in S(A),\forall g\in S(A)^\perp,\qquad \Omega(f+g)\geq\max\{\Omega(f),\Omega(g)\}\]
\end{definition}


Consider the subspace valued map $S_{\mathbb{R}}$ from Example \ref{ex:subvalmaps}.
\cite[Theorem 1]{DinuzzoS2012} showed that a functional $\Omega$ is orthomonotone with respect to $S_\mathbb{R}$ if and only if there exists a monotonically increasing functional $h:\mathbb{R}\to\mathbb{R}\cup\{\infty\}$ such that $\Omega(z)=h(||z||), \forall z\in\mathcal{Z}$.
Note that while the above characterization with a monotonically increasing functional restricts its analysis to inner product induced norms, other kinds of orthomonotone functionals can be constructed as well, as shown in the examples below.
\begin{example}{{Orthomonotone functionals}}
	\label{ex:ortho}
	\begin{enumerate}
		\item $\Omega(z) = ||z||^p_{\mathcal{Z}}$, for any $p>0$ is orthomonotone w.r.t. $S_{\mathbb{R}}$
		\item Let $\mathcal{Z}=\mathbb{R}^n$ and $||\cdot||_{1}$ denote the $\ell_1$ norm. Then, $\Omega(z) = ||z||_1$ is orthomonotone w.r.t. $S_{proj}$ ($S_{proj}$ as defined in Example \ref{ex:subvalmaps}).
		\item Consider the space of differentiable functions $\mathcal{H}$ from Example \ref{ex:derOp} and the differential operator $D:\mathcal{H}\to\mathcal{Y}$ defined therein. Then $\Omega(f)=||Df||_{\mathcal{Y}}^2$ is orthomonotone with respect to $S_D:=D^*\circ S_{proj}\circ D$
	\end{enumerate}
%
\end{example}
The proof for the first statement follows directly from \cite[Theorem 1]{DinuzzoS2012} since $\Omega(z) = ||z||^p_{\mathcal{Z}}$, for any $p>0$ is a monotonically increasing function of the inner product induced norm.
The proof for the second statement follows from Theorem \ref{thm:l1ortho} below and
the third statement follows from Theorem \ref{thm:orthoComp}.

Note that the second statement in the example above shows how sparse regularization problems involving the $\ell_1$ norm are also covered by the notion of orthomonotone functionals. The third statement shows the ability to regularize after composition with linear operators that have a non trivial null space.

The orthomonotonicity of $\ell_1$ regularizers is formalized with the following theorem,

\begin{theorem}{{Orthomonotonicity of $\ell_1$ regularizers}\\}
	\label{thm:l1ortho}
	Let $\mathcal{Z}=\mathbb{R}^n$, $S_{proj}$ be the subspace valued map defined in Example \ref{ex:subvalmaps} and let $h:[0,\infty]\to\mathbb{R}\cup\{+\infty\}$ be a monotonic increasing function. Then $\Omega(z)=h(||z||_1)$ is orthomonotone with respect to $S_{proj}$.
	\begin{proof}
		We first show $\Omega(z)=||z||_1$ is orthomonotone w.r.t. $S_{proj}$. The result for monotonic increasing $h$ follows from there.
		
		Let $E=\{e_1,\dots,e_n\}$ be the standard basis for $\mathbb{R}^n$. 
		Note that for any $z\in\mathbb{R}^n$, $S_{proj}(z)=\{\sum_{i=1}^{n}\lambda_i\langle z,e_i \rangle_{\mathbb{R}^n}e_i:e_i\in E,\lambda_i\in\mathbb{R}\}$ and $(S_{proj}(z))^\perp=\{\sum_{j}\lambda_j e_j:\langle z,e_j \rangle_{\mathbb{R}^n}=0,e_j\in E,\lambda_j\in\mathbb{R} \}$. Similarly for a set $A\subset \mathbb{R}^n$, $S_{proj}(A)=\{\sum_{i=1}^{n}\lambda_i\langle z,e_i \rangle_{\mathbb{R}^n}e_i:e_i\in E,\lambda_i\in\mathbb{R},z\in A \}$ and $(S_{proj}(A))^\perp=\{\sum_j\lambda_j e_j: e_j\in E,\lambda_j\in\mathbb{R}, \forall z\in A, \langle z,e_j \rangle_{\mathbb{R}^n}=0 \}$. Now for any $z\in S_{proj}(A)$ and $c\in S_{proj}(A)^\perp$, $|| z+c ||_1 = \sum_{\{i:\langle z,e_i \rangle_{\mathbb{R}^n}\neq 0  \}}|z_i| + \sum_{\{i:\langle z,e_i \rangle_{\mathbb{R}^n}= 0  \}}|c_i|$ with $z_i = \langle z,e_i \rangle_{\mathbb{R}^n}$ and $c_i=\langle c,e_i \rangle_{\mathbb{R}^n}$. Also $||z||_1 = \sum_{i=1}^{n} |z_i| = \sum_{\{i:\langle z,e_i \rangle_{\mathbb{R}^n}\neq 0  \}}|z_i| $ and $||c||_1 = \sum_{i=1}^{n} |c_i| = \sum_{\{i:\langle z,e_i \rangle_{\mathbb{R}^n}= 0  \}}|c_i|$. Thus we see $|| z+ c||_1 = ||z||_1+||c||_1 \geq \max\{||z||_1,||c||_1\}$ $\implies$ $\Omega(z)=||z||_1$ is orthomonotone with respect to $S_{proj}$.

		Now for any monotonically increasing function $h$, for any $a,b\in[0,\infty)$, $a>b$ implies $h(a)>h(b)$. Thus $||z+c||_1\geq \max\{||z||_1,||c||_1\}$ implies $h(||z+c||_1)\geq \max\{h(||z||_1),h(||c||_1)\}$. And thus $\Omega(z)=h(||z||_1)$ is orthomonotone with respect to $S_{proj}$ for any monotonically increasing function $h$.
	\end{proof}
\end{theorem}
The third statement in example \ref{ex:ortho}, follows from the theorem below, 
\begin{theorem}{Orthomonotone functionals composed with Linear Operators\\}
	\label{thm:orthoComp}
	If $\Omega$ is orthomonotone with respect to a closed, super additive subspace valued map $S$ and $L:\mathcal{H}\to \mathcal{Z}$ preserves orthogonality with respect to $S_L$, then $\Omega \circ L(f):= \Omega(Lf)$, is orthomonotone with respect to $S_L=L^*\circ S \circ L$, i.e, $\forall A\in \mathbb{F}_\mathcal{H}, f\in S_L(A), g\in S_L(A)^\perp$, $\Omega(Lf + Lg)\geq \max\{\Omega(Lf),\Omega(Lg)\}$.
	\label{thm:OperatorComposition}
	\begin{proof}
		Since $L$ preserves orthogonality with respect to $S_L$, we know $LS_L(A)\subseteq S(LA)$ and $L(S_L(A)^\perp)\subseteq S(LA)^\perp$. Thus $\forall f\in S_L(A), g\in S_L(A)^\perp$, $Lf\in S(LA), Lg\in S(LA)^\perp$. Then by orthomonotone property of $\Omega$ with respect to $S$, we must have $\Omega(Lf+Lg)\geq \max\{\Omega(Lf),\Omega(Lg)\}$ $\implies$ $\Omega\circ L$ is orthomonotone w.r.t. $S_L$.
	\end{proof}
\end{theorem}
With the notions of Linear and Adjoint operators combined with Subspace Valued maps and Orthomonotone functionals, we are now ready to present the main result for the Generalized Representer Theorem.

\section{Generalized Representer Theorem}
\label{sec:GRT}
	Let $\mathcal{H}$ be an arbitrary Hilbert space.  
	For any $m\in\mathbb{N}$ and $i\in \{1,\dots,m+1\}$, let $L_i:\mathcal{H}\to\mathcal{Z}_i$ be continuous linear operators from $\mathcal{H}$ to arbitrary Hilbert spaces $\mathcal{Z}_i$. Let the Hilbert space obtained from $\mathcal{Z}_1\times\mathcal{Z}_2\times\dots\mathcal{Z}_{m}$ be denoted $\mathcal{Z}$ and let $\mathbb{F}_\mathcal{H},\mathbb{F}_{\mathcal{Z}_i},\mathbb{F}_{\mathcal{Z}}$ be the power set of $\mathcal{H}$, $\mathcal{Z}_i$ and $\mathcal{Z}$ respectively. Let $C:\mathcal{Z}\to \mathbb{R}\cup\{+\infty\}$ and $\Omega:\mathcal{Z}_{m+1}\to \mathbb{R}\cup\{+\infty\}$ be some lower semi-continuous functionals.
	
	Consider the functional $J:\mathcal{H}\to\mathbb{R}\cup\{+\infty\}$,
	\begin{equation}
	J(f):= C(L_1f,\dots,L_mf)+\Omega(L_{m+1} f) \label{eq:Reg3}
	\end{equation}
	Given a functional $J$ specified by $(C,\Omega,L_1,\dots,L_{m+1})$, a learning problem is then posed as $$f_{opt}=\underset{f\in\mathcal{H}}{\text{argmin}} \quad J(f)$$
	 The inclusion of $\{+\infty\}$ in the range of lower semi-continuous $C$ and $\Omega$ allows one to consider constrained optimization problems. Following are a few examples of learning problems written in this form,
	 \begin{example}{Learning problems\\}
	 	\begin{enumerate}
	 		\item Let $\mathcal{H}$ be an RKHS space of functions taking values in $\mathcal{Z}=\mathbb{R}^n$. Consider the evaluation operator from Example \ref{ex:EvalOp} such that $L_x:\mathcal{H}\to\mathcal{Z}$ is given by $L_xf:=f(x)$. Let $\{(x_i,y_i):i=1,\dots,m\}$ be a training data set. Let $L_1,\dots,L_m$ be given by $L_{x_1},\dots,L_{x_m}$ and 
	 		$L_{m+1}:\mathcal{H}\to\mathcal{H}$ be the identity operator. Let $C(L_1f,\dots,L_mf) := \sum_{i=1}^{m}|| y_i - \sigma(L_{x_i}f) ||_\mathcal{Z}^2$ for some activation function $\sigma:\mathbb{R}^n\to\mathbb{R}^n$. Let $\Omega(L_{m+1}f):=||f||_{\mathcal{H}}^2$.
	 		Then for $J(f)=\sum_{i=1}^{m}|| y_i - \sigma(L_{x_i}f) ||_\mathcal{Z}^2+||f||_{\mathcal{H}}^2$ we get a regularized least squares problem in the RKHS space if $\sigma$ is linear and an RKHS based neural network layer for some nonlinear $\sigma$.
	 		\item Let $\Omega(f)=||f||_1^2$ in the above example and we get a $\ell_1$ regularized problem.
	 		\item Let $\mathcal{Z}=\mathbb{R}$, $y_i\in\{+1,-1\}$, $C(L_1f,\dots,L_mf):= \begin{cases}
			0	& 		\forall i\in\{1,\dots,m\}; \quad y_i L_if >0\\
			+\infty & \text{otherwise}
	 		\end{cases}$ and $\Omega(f)=||f||^2$. Then $J(f)=C(L_1f,\dots,L_mf)+\Omega(f)$ gives a Support Vector Machine for binary classification.
	 	\end{enumerate}
	 \end{example}
 Given a learning problem in terms of a functional $J$, we can next define the notion of a linearly representable problem.
 \begin{definition}{Linearly Representable Problem\\}
 	Consider the functional $J:\mathcal{H}\to\mathbb{R}\cup\{+\infty\}$ from \eqref{eq:Reg3}.
 	Let $S:\mathbb{F}_{\mathcal{Z}_{m+1}}\to\mathbb{F}_{\mathcal{Z}_{m+1}}$ be a subspace valued map. Let $S_{L_{m+1}}:\mathbb{F}_{\mathcal{H}}\to\mathbb{F}_{\mathcal{H}}$ be $S_{L_{m+1}} := L_{m+1}^*\circ S \circ L_{m+1}$ as given by Propostion \ref{prop:pullbackS}. Let $\mathcal{N}_{L_1,\dots,L_m}^\perp$ be the orthogonal complement to the joint null space for the operators $L_1,\dots,L_m$ as given by corollary \ref{cor:nullspaceorth} and let $L_{m+1}$ preserve orthogonality w.r.t. $S_{L_{m+1}}$ (as defined in Definition \ref{def:orhtpresL}).
 	
 	The functional $J$ is said to be Linearly Representable with respect to $S$ if a minimizer for $\underset{f\in\mathcal{H}}{\text{min}} \quad J(f)$ exists in $S_{L_{m+1}}(\mathcal{N}_{L_1,\dots,L_m}^\perp)$.
 	
 	 Further a family of functionals $\mathcal{F}$ is said to be \textbf{Linearly Representable with respect to $S$} if every $J\in\mathcal{F}$ is Linearly representable with respect to $S$.
 \end{definition}

The notion of linear representability is quite significant as it allows one to write the minimizer in a possibly infinite dimensional space $\mathcal{H}$ in terms of finitely many vectors spanning $\mathcal{N}_{L_1,\dots,L_m}^\perp$. This often allows one to reformulate infinite dimensional optimization problems in $\mathcal{H}$ into equivalent finite dimensional optimization in $\mathcal{Z}$.

The Generalized Representer Theorem provides necessary and sufficient conditions for a family of functionals $\mathcal{F}$ to be Linearly Representable.
Below we state and prove, first the sufficient condition for Linear Representability of a functional $J$ and then the complete statement of necessary and sufficient condition for a family of functionals $\mathcal{F}$.

\begin{theorem}{Generalized Representer Theorem (Sufficient condition)\\}
	\label{lem:Sufficiency}
	Let $S$ be an inclusive, closed, super additive subspace valued map. For any $J$ of the form \eqref{eq:Reg3} with $\Omega$ and $C$ lower semi-continuous, the functional $J$ is Linear representable with respect to $S$,  if $\Omega$ is orthomonotone with respect to $S$.
	\begin{proof}
		Let $A = \mathcal{N}_{L_1\dots L_m}^\perp$. If $\Omega$ is orthomonotone w.r.t. $S$ then $\Omega\circ L_{m+1}$ is orthomonotone w.r.t. $S_{L_{m+1}}$ (by Theorem \ref{thm:OperatorComposition}). Thus $\forall f\in S_{L_{m+1}}(A), g\in S_{L_{m+1}}(A)^\perp$, $\Omega(L_{m+1}f+L_{m+1}g)\geq \Omega(L_{m+1}f)$.
		Also, if $S$ is inclusive, closed and super additive, so is $S_{L_{m+1}}$. And thus by Lemma \ref{lem:closedQuasInc}, $S_{L_{m+1}}$ is null space preserving with respect to $\{L_1,\dots,L_m\}$, i.e., $S_{L_{m+1}}(A)^\perp \subseteq \mathcal{N}_{L_1,\dots,L_m}$. Thus for all $g\in S_{L_{m+1}}(A)^\perp$, $L_ig=0$ for all $i\in\{1,\dots,m\}$.
		
		Now, note that $S_{L_{m+1}}(A)$ and $S_{L_{m+1}}(A)^\perp$ forms an orthogonal complementary pair for $\mathcal{H}$, thus for any $F \in \mathcal{H}$ we can find a decomposition for $F = f+g$, $f\in S_{L_{m+1}}(A)$, $g\in S_{L_{m+1}}(A)^\perp$. 
		Then
		\begin{eqnarray}
			J(F) &=& C(L_1(f+g),\cdots,L_m(f+g)) + \Omega(L_{m+1}(f+g)) \\
			&=& C(L_1f,\cdots,L_mf)+\Omega(L_{m+1}f+L_{m+1}g) \\
			&\geq& C(L_1f,\cdots,L_mf)+\Omega(L_{m+1}f)
		\end{eqnarray} 
		Thus $\forall F\in\mathcal{H}$, $\exists f\in S_{L_{m+1}}(A)$ such that $J(f)\leq J(F)$. Thus if $J$ admits a minimizer in $\mathcal{H}$, a minimizer must exists in $S_{L_{m+1}}(A)$, implying $J$ is Linearly Representable w.r.t. $S$.
	\end{proof}
\end{theorem}

The Generalized Representer Theorem we present here differs from its prior counterpart \cite[Theorem 3.1]{Argyriou:2014:UVR:3044805.3044976} in two significant ways. Firstly, there is no assumption for a finite dimensional $r$-regularity property on the subspace valued map and secondly, the output space $\mathcal{Z}$ can be arbitrary infinite dimensional Hilbert spaces. These two changes become significant since when dealing with stochastic regression problems the output space  $\mathcal{Z}$ is an infinite dimensional semi-Hilbert space of random variables and when dealing with $\ell_1$ regularization problems in function spaces, the corresponding subspace valued map $S_{proj}$ is not $r$-regular for any finite $r$. We will expand upon these differences in Section \ref{sec:examples} with corresponding application examples.

To prove the necessary part of the theorem, first consider the following proposition.
\begin{proposition}
	\label{prop:familyL}
	Let $z^\star= (z^\star_1,\dots,z^\star_m)$ be a minimizer for $C$. 
	Let $A=\mathcal{N}_{L_1\dots L_m}^\perp$ and $f\in S_{L_{m+1}}(A)$, $f\neq 0$.
	Then for a collection of $m$ linear operators $L'_i:\mathcal{H}\to\mathcal{Z}_i$ such that $L'_i h=z^\star_i \langle f,h \rangle_\mathcal{H}/||f||^2$ for $i = 1,\dots,m$. 
	\begin{enumerate}
		\item $g\in S_{L_{m+1}}(A)^\perp$ $\implies$ $L'_ig = 0$ 
		\item $L'_if = z_i^\star$ and the adjoint is given by ${L'_i}^*z = \langle z_i^\star, z\rangle_{\mathcal{Z}_i} f$
	\end{enumerate}
\begin{proof}
	Note that $g\in S_{L_{m+1}}(A)^\perp$, $f\in S_{L_{m+1}}(A)$ implies $g\perp f$ and $L'_ig = (z_i^\star/||f||^2) \langle f,g\rangle_{\mathcal{H}} = 0$. Thus showing the first statement of the proposition. 
	$L'_if=z_i^\star$ follows by substituting $f$ into the definition for $L'_if$. For the adjoint, note that ${L'_i}^*$ is such that $\langle {L'_i}^*z,h \rangle_\mathcal{H} = \langle z, L'_ih \rangle_{\mathcal{Z}_i} = \langle z, z_i^\star \langle f,h \rangle_{\mathcal{H}}/||f||^2 \rangle_{\mathcal{Z}_i} = \langle z, z_i^\star \rangle_{\mathcal{Z}_i} \langle f,h \rangle_{\mathcal{H}}/||f||^2$. Thus we can conclude ${L'_i}^*z = \langle z_i^\star, z\rangle_{\mathcal{Z}_i} f/||f||^2$. 
\end{proof}
\end{proposition}
The above proposition shows the existence of a nonempty subspace of linear operators $(L_1',\dots,L_m')$ such that a fixed $(S_{L_{m+1}}(\mathcal{N}_{L_1,\dots,L_m}^\perp))^\perp\subseteq \mathcal{N}_{L_1',\dots,L_m'}$ and $S$ is null space preserving with respect to $\{L_1',\dots,L_m'
\}$ (by Lemma \ref{lem:closedQuasInc}).
\begin{definition}{(Null space preserving operators)\\}
	Let $S$ be a closed, super additive subspace valued map. Then the space of continuous linear operators $\mathcal{L}=\{(L_1,\dots,L_m):  S \text{ is null space preserving w.r.t. } \{L_1,\dots,L_m\} \}$ is called a family of null space preserving operators with respect to $S$.
\end{definition}
Note that the operators $(L_1',\dots,L_m')$ from Proposition \ref{prop:familyL} belong to $\mathcal{L}$.
\begin{definition}{(Family of null space preserving functionals)\\}
	\label{def:nullspacepresFamily}
	Let $S$ be an inclusive, closed, super additive subspace valued map. Let $C,\Omega$ be lower semicontinuous functionals and $L_{m+1}:\mathcal{H}\to\mathcal{Z}_{m+1}$ and $\mathcal{L}$ be a family of null space preserving operators w.r.t. $S_{L_{m+1}}$. Then consider the family of functionals
	$
		\mathcal{F}:=\{J:\mathcal{H}\to\mathbb{R}\cup\{+\infty\}: J(f)=C(L_1f,\dots,L_mf)+\gamma\Omega(L_{m+1}f), \gamma\in[0,\infty), (L_1,\dots,L_m)\in\mathcal{L} \}
	$. We will call this a family of null space preserving functionals.
\end{definition}
\begin{theorem}{Generalized Representer Theorem (Necessary and Sufficient Conditions) \\}
	\label{thm:NecSuff}
	A family of null space preserving functionals $\mathcal{F}$ is Linearly Representable \textbf{if and only if}, $\Omega\circ L_{m+1}$ is orthomonotone with respect to $S_{L_{m+1}}$ \label{cnd:OmegaCond}
	\begin{proof}
		The proof for sufficiency (i.e. orthomonotone $\Omega$ $\implies$ existence of representer theorem) follows from Theorem \ref{lem:Sufficiency}.
		
		To prove necessity of orthomonotone $\Omega$, let the family of functionals $\mathcal{F}$ be linear representable w.r.t. to map $S$.
		
		 Consider a functional $J:=(C,\Omega,L_1,\dots,L_m,L_{m+1})\in\mathcal{F}$ and construct a functional  $J':=(C,\Omega,L'_1,\dots,L'_m,L_{m+1})\in\mathcal{F}$ with $L'_1,\dots,L'_m$ as given in Proposition \ref{prop:familyL}. Let $f_{J'}$ be the minimizer for $J'$. 
		
		Further we know $C(z_1^\star,\dots,z_m^\star)+\Omega(L_{m+1}f_{J'})\leq C({L'_1}f_{J'},\dots,{L'_m}f_{J'})+\Omega(L_{m+1}f_{J'})\leq C({L'_1}(f+g),\dots,{L'_m}(f+g) + \Omega(L_{m+1}(f+g))$ for any $f,g\in\mathcal{H}$. For $A = \mathcal{N}_{L_1\dots L_m}^\perp$,  consider $f\in S_{L_{m+1}}(A)$, $g\in S_{L_{m+1}}(A)^\perp$, then $L'_ig=0$ and $L'_if=z_i^\star$. Thus we have 
		$C(z_1^\star,\dots,z_m^\star)+\Omega(L_{m+1}f_{J'})\leq C({L'_1}f_{J'},\dots,{L'_m}f_{J'})+\Omega(L_{m+1}f_{J'})\leq C(z_1^\star,\dots,z_m^\star) + \Omega(L_{m+1}(f+g))$ $\implies$ $\Omega(L_{m+1}f_{J'})\leq \Omega(L_{m+1}(f+g))$.

		Now if we consider a Cauchy sequence $\gamma_k \to 0$ converging to $0$ and a sequence of functionals $J'_k:=(C,\gamma_k \Omega,L'_1,\dots,L'_m,L_{m+1})\in\mathcal{F}$. Then we get a sequence of minimizers $f_{J'_k}\to f$. Also since $\Omega(L_{m+1}f_{J'_k})\leq \Omega(L_{m+1}(f+g))$ for all $f_{J'_k}$, this implies $\Omega(L_{m+1}f)\leq \Omega(L_{m+1}(f+g))$ for all $f\in S_{L_{m+1}}(A),g\in S_{L_{m+1}}(A)^\perp$. 
		
		Similarly for $w\in S_{L_{m+1}}(A)^\perp$, $w\neq 0$ consider, operators $L_i''h=z_i^\star\langle w,h\rangle_{\mathcal{H}}/||w||_{\mathcal{H}}^2$. Note that for all $f\in S_{L_{m+1}}(A)$, $L_i''f=0$ and for all $g\in S_{L_{m+1}}(A)^\perp$, $L_i''g=z^\star$. Also $(L_1'',\dots,L_m'')\in \mathcal{L}$ by Lemma \ref{lem:closedQuasInc}. Consider the functional $J''=(C,\Omega,L_1'',\dots,L_m'',L_{m+1})\in\mathcal{F}$ and let the minimizer for $J''$ be $f_{J''}$. Then as before $C(z_1^\star,\dots,z_m^\star)+\Omega(L_{m+1}f_{J''})\leq C({L''_1}f_{J''},\dots,{L''_m}f_{J''})+\Omega(L_{m+1}f_{J''})\leq C({L''_1}(f+g),\dots,{L''_m}(f+g) + \Omega(L_{m+1}(f+g))$. Thus $\Omega(L_{m+1}f_{J''})\leq\Omega(L_{m+1}(f+g))$. Considering then the sequence of functionals $J''_k:=(C,\gamma_k \Omega,L''_1,\dots,L''_m,L_{m+1})\in\mathcal{F}$, the corresponding minimizers $f_{J''_k}\to g$ as $\gamma_k\to 0$. Thus for all $g\in S_{L_{m+1}}(A)^\perp$, $\Omega(L_{m+1}g)\leq \Omega(L_{m+1}(f+g))$.
		
		Thus the existence of linear representers for the family $\mathcal{F}$ implies for all $f\in S_{L_{m+1}}(A)$, $g\in S_{L_{m+1}}(A)^\perp$, $ \Omega(L_{m+1}(f+g))\geq \max\{\Omega(L_{m+1}f),\Omega(L_{m+1}g)\}$, i.e. $\Omega\circ L_{m+1}$ is orthomonotone w.r.t. $S_{L_{m+1}}$.
	\end{proof}
\end{theorem}
\begin{remark}{(Extension and Previous Works)\\}
	We presented here a generalized version of Representer theorems for Hilbert space valued functions with general loss functions on an arbitrary target Hilbert space $\mathcal{Z}$ without the assumption of ``r-regularity" to allow for more general regularization like the $\ell_1$ norm.
	
	Special cases of the theorem addressing least squares regularization for vector valued functions in Reproducing Kernel Hilbert Space (RKHS) framework can be found in \cite[Theorems 3.1, 4.1]{Micchelli:2005:LVF:1119610.1119620}. Special cases of the theorem for $\ell_1$ regularization can be found in \cite{unser2016representer}. A generalized version of the Representer theorems for more general loss functions but still restricted to RKHS of real valued functions can be found in \cite{DinuzzoS2012,Scholkopf:2001:GRT:648300.755324}. The far more general framework of subspace valued maps was introduced in \cite[Theorem 3.1]{Argyriou:2014:UVR:3044805.3044976}. 
	
	\cite{Argyriou:2014:UVR:3044805.3044976} however restricts its loss function to the form $J(f):= C(\langle f,w_1\rangle,\dots, \langle f,w_m \rangle) + \Omega(f) $ where $C$ necessarily takes arguments from $\mathbb{R}^m$. We extend this result to allow arguments for $C$ and $\Omega$ in an arbitrary Hilbert space $\mathcal{Z}$. Considering an arbitrary Hilbert space $\mathcal{Z}$ for the output also has the effect that representer theorems for vector valued outputs can be simply explained away with the $S_{\mathbb{R}}$ subspace valued map as opposed to using a matrix based $S_{\mathcal{L}}$ subspace valued maps as was required by \cite{Argyriou:2014:UVR:3044805.3044976} ($S_{\mathbb{R}},S_{\mathcal{L}}$ as defined in Example \ref{ex:subvalmaps}). Further with infinite dimensional outputs $\mathcal{Z}$, cases of such outputs occurring in Bayesian regression settings (special case in \cite{pillai2007characterizing}) can also be tackled, which were outside the scope of previous works.
\end{remark}

\section{Application Examples}
\label{sec:examples}
\subsection{Deep Neural Networks}
\begin{figure}
	\centering
	\centering
	\includegraphics[width=0.8\textwidth,trim={2cm 1.8cm 2cm 1.8cm},clip]{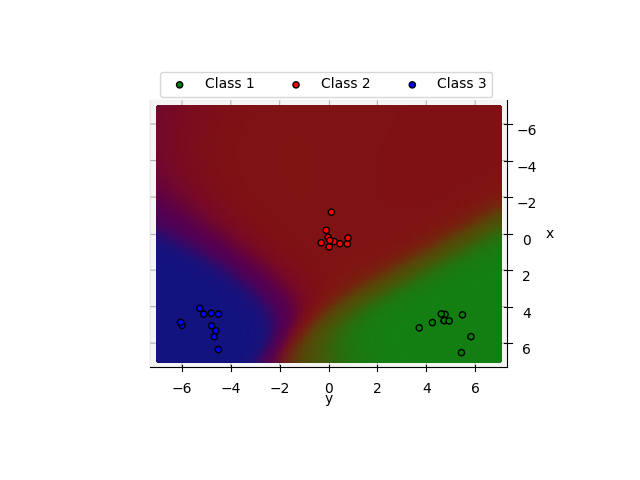}
	\caption{Multi class classification with a 3 layer squared exponential kernel based neural network. Class probabilities shaded as red, blue, green values. \\
		Training data shown as point clusters.}
	\label{fig:nn}
\end{figure}
Consider a single layer perceptron with an activation function $\sigma$, with input $x$, hidden variables $h = f(x)$ and output $y$. Given $m$ training samples $\{(x_i,y_i):i\in\mathbb{N}_m\}$ consider the variational learning problem
\begin{equation}
\underset{f\in\mathcal{H}}{\text{min}} \quad \sum_{i=1}^{m} || y_i - \sigma(L_{x_i}f) ||_\mathcal{Z}^2 + \lambda || f ||_\mathcal{H}^2
\end{equation}
This minimization problem fits exactly the form of \eqref{eq:Reg3} by taking $C$ to be $\sum_{i=1}^{m}||y_i - \sigma(L_{x_i}(\cdot)) ||^2$ and $\Omega$ to be $|| f ||_\mathcal{H}^2$. Since $\Omega$ is orthomonotone with respect to $S_{\mathbb{R}}$, we know a minimizer of the form $\sum_{i=1}^{m} L_{x_i}^* z_i$ must exist. Substituting this form into the minimization above we can get a finite dimensional minimization problem. Also for $\mathcal{H}$ restricted to an RKHS we know the adjoint $L_{x}^*$ to be the kernel section from Example \ref{ex:EvalOp}. For formulations with an explicit basis we know $L_x^*$ from Example \ref{ex:finiteexpbasis}. Thus we have a nonlinear program to solve for a kernel based and explicit basis based deep neural network with $z_i\in\mathcal{Z}$ being the new decision variables. Note that the program becomes nonlinear due to a nonlinear activation function $\sigma$ and only thus differs from a generalized least squares setting.

Now for a N-layer perceptron, consider each layer perceptron to be given by $f^{(l)}\in\mathcal{H}^{(l)}$, $C_l$, $\Omega_l$, $\sigma_l$ and output $y^{(l)}$, $l\in \{1,\dots,N\}$. Also lets denote the inputs $x_i$ as $y_i^{(0)}$ and observed output $y_i$ as $y_i^{(N)}$ for notational convenience.
Then consider the minimization problem
\begin{equation}
\label{eq:nnopt}
\underset{\{y^{(l)}:l=1,\dots,N-1\},\{f^{(l)}\in\mathcal{H}^{(l)}:l=1,\dots,N\}}{\text{min}} \quad \sum_{l=0}^{N-1}\left[\sum_{i=1}^{m} || y_i^{(l+1)} - \sigma_l(L_{y_i^{(l)}}f^{(l)}) ||_{\mathcal{Z}^{(l)}}^2 + \lambda_l || f ||_{\mathcal{H}^{(l)}}^2\right]
\end{equation}

One can notice here the similarity of the above problem to the discrete time multiple shooting problems in numerical optimal control where optimal decisions are to be made over a N step horizon and $y_i^{(l)}$s are the predicted states of the system to be solved for. The key idea in multiple shooting methods is to find the optimal solution for each segment $(l)\to(l+1)$ assuming a fixed $\bar{y}^{(l)}$ is given and then impose the additional constraint $\bar{y}^{(l+1)} = \sigma(L_{\bar{y}^{(l)}}f^{(l)})$ for $l\in\{0,\dots,N-1\}$. 

Thus for any fixed set $\{y_i^{(l)}:i=1,\dots,m, l=0,\dots,N\}$, we know a minimizer for $f^{(l)}$ will take the form $f^{(l)} = \sum_{i=1}^{m}L_{y_i^{(l)}}^*z^{(l)}_i$. One thus reduces the above problem to a finite dimensional nonlinear program in $y^{(l)}_i,z^{(l)}_i$. Solving it like a multiple shooting problem with each segment minimized and then a consensus constraint on the hidden variables also make the problem highly parallelizable. Below we show an example of a 3 layer neural network with $\mathcal{H}^{(l)}$ being an RKHS space with a squared exponential kernel of functions from $\mathbb{R}^2\to\mathbb{R}^3$. The neural network is used as a 3-class classifier. Inputs $x_i$ are points from a point cloud in $\mathbb{R}^2$ and the outputs $y_i$ are class labels encoded as a one hot encoding, $y_i=(1,0,0)$ for class 1, $y_i=(0,1,0)$ for class 2 and $y_i=(0,0,1)$ for class 3. Then starting with a random guess for $\{y_i^{(l)}:i\in\{1,2\}\}$ and $\{z_i^{(l)}:i\in\{1,2,3\}\}$, we solve the optimization in \eqref{eq:nnopt} with repeated optimizations tightening the constraint towards $\bar{y}^{(l)} = \sigma(L_{\bar{y}^{(l-1)}})$. Passing the output predictions of the network through a logistic function, gives us a probability for any point in $\mathbb{R}^2$ to be in class 1,2 or 3. A logistic soft-max function is used to label the predictions. Figure \ref{fig:nn} shows the output of the trained neural network with class probability for points in $\mathbb{R}^2$ shaded with corresponding RGB color values.

\subsection{Learning Stochastic Processes}
Let $(\Omega,\mathcal{F},\mathbb{P})$ be a probability measure space. Consider a family of Hilbert spaces $\mathcal{G}=\{\mathcal{G}_{\omega}:\omega\in \mathcal{F}\}$, in which for each $\omega\in\mathcal{F}$, $\mathcal{G}_{\omega}$ is a Hilbert space of deterministic functions $f_{\omega}:\mathcal{X}\to \mathcal{R}_\omega$ taking members of an index set $\mathcal{X}$ to a deterministic Hilbert space $\mathcal{R}_\omega$ of vectors in $\mathbb{R}^n$.
$\mathcal{G}$ forms a semi-Hilbert space of Stochastic Processes such that $\langle f,g \rangle_\mathcal{G}:= \mathbb{E}[\langle f(\cdot,\omega),g(\cdot,\omega)\rangle_{\mathcal{G}_\omega} ]$. Define an equivalence relation $\sim$ which says $f\sim g$ if $\langle f-g,f-g \rangle_\mathcal{G}=0$. The quotient space $\mathcal{H}=\mathcal{G}\backslash \sim$ then defines a Hilbert space where all equivalent processes are considered as a single element in the space. Similarly defining $\mathcal{R}=\{\mathcal{R}_\omega:\omega\in\mathcal{F}\}$ with inner product $\langle z_1,z_2 \rangle_\mathcal{R}= \mathbb{E}[\langle z_1(\omega),z_2(\omega) \rangle_\mathcal{R}]$ and an equivalence relation $z_1\sim z_2$ if $\langle z_1-z_2,z_1-z_2 \rangle_{\mathcal{R}}=0$. We get a Hilbert space $\mathcal{Z}=\mathcal{R}\backslash \sim$ of $n$-dimensional random vectors.

Now, consider a parametric evaluation operator $L_x:\mathcal{H}\to\mathcal{Z}$ defined as $L_x f:= f(x,\cdot)$ where $L_xf$ maps $f$ to a Gaussian random vector in $\mathcal{Z}$. 
A classical additive Gaussian noise observation model is $y = f(x)+\eta$ with $\eta \sim \mathcal{N}(0,\Sigma_\eta)$. Thus $y=L_xf + \eta$ maps $f$ to a Gaussian random observation vector $y\in\mathcal{Z}$ if $\mathcal{H}$ is a space of Gaussian processes. 

The adjoint $L_x^*$ can then be specified by observing that $\langle L_x^*z,f\rangle_\mathcal{H} =\mathbb{E}[\langle L_{x,\omega}^*z(\omega),f(\cdot,\omega)\rangle_{\mathcal{G}_\omega}] = \langle z,L_x f\rangle_\mathcal{Z} = \mathbb{E}[z(\omega)^T f(x,\omega)]$. Then if we restrict $\mathcal{G}_\omega$ to be a RKHS with kernel $K_\omega(\cdot,\cdot)$, the adjoint action $L^*_{x,\omega}z(\omega) = K_\omega(x,\cdot)z$ maps the random vector $z$ to the random process $K_\omega(x,\cdot)z(\omega)$ for random events $\omega$. A special case of the above would be to consider a common RKHS $\mathcal{G}$ with kernel $K$ for all $\omega$, then the adjoint $L^*_{x}z = K(x,\cdot)z$.

Now with the spaces and adjoint defined we can consider a regression problem with $\mathcal{H}$ being a RKHS for Gaussian processes with kernel $K$, and $\mathcal{Z}$ being the space of $n$ dimensional Gaussian random vectors.
\begin{equation}
\underset{f\in\mathcal{H}}{\text{min}} \quad \sum_{i=1}^{m} || y_i - L_{x_i}f - \eta_i ||_\mathcal{Z}^2 + \lambda || f ||_\mathcal{H}^2
\end{equation}

Here the functionals $C(L_1f,\dots,L_mf), \Omega(f)$ are strictly convex and orthomonotone with respect to the subspace valued map $S_{\mathbb{R}}$. From Theorem \ref{thm:NecSuff} we know a linear representer w.r.t. $S$ must exist for a minimizer.

Thus a unique minimizer of the form $f_{opt} \in \{\sum_{i=1}^{m} K({x_i},\cdot) z_i:z_i\in\mathcal{Z}\}$ exists. Substituting for $f_{opt}$ into the minimization problem we can now get a finite dimensional minimization problem with decision variables being the mean and variances of $z_i$. 

Let $\bar{z}_i^\star$ be the mean for optimal $z_i$ and $C_{z_{ij}}^\star$ be the covariance between the optimal $z_i,z_j$.
The mean and variance functions for the process can then be written as $\bar{f}(x) = \sum_{i=1}^{m} K(x_i,x)\bar{z}_i^\star$ and $\Sigma_{f}(x) = \sum_{i,j=1}^{m}K(x_i,x)C_{z_{ij}}^\star K(x,x_j)$.

\begin{figure}
	\centering
	\centering
	\includegraphics[width=0.7\textwidth]{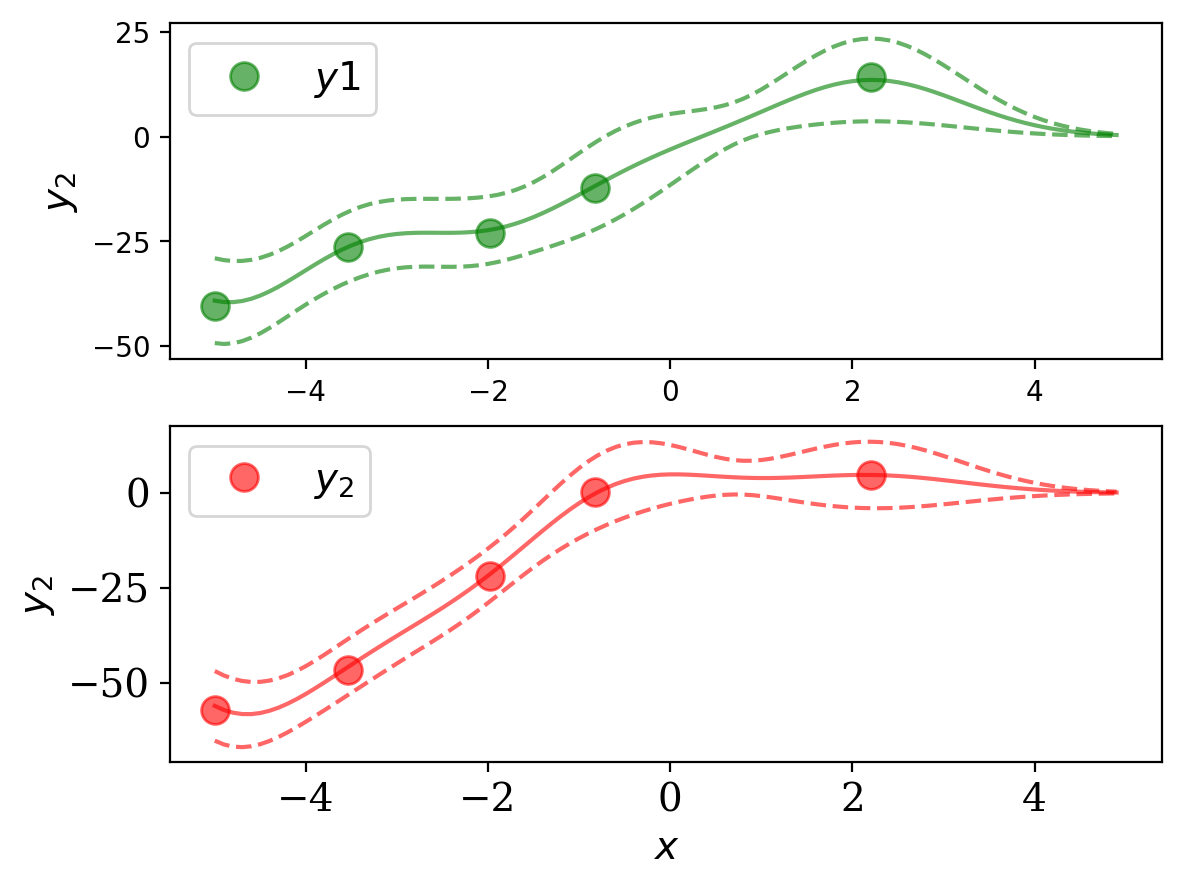}
	\caption{Least squares regression in Gaussian process space}
	\label{fig:GaussianRegression}	
\end{figure}

Figure \ref{fig:GaussianRegression} shows an example for such a regression with a squared exponential kernel mapping with the output $y_i\in \mathcal{Z}$ being a two dimensional Gaussian random vector and $x_i\in\mathbb{R}$. It should be noted that by writing the KKT conditions for optimality, it can be verified that the mean prediction coincides with the classical result for Bayesian prediction, with $z_i^\star$ being given as a solution to the linear system of equations
\[
\sum_{k=1}^{m} (\lambda \delta_{ik} +K(x_i,x_k))\bar{z}_k = \bar{y}_i \qquad \forall i\in \mathbb{N}_m
\] where $\delta_{ik}=1$ if $i=k$ and $0$ otherwise.

\subsection{$\ell_1$-Regularization}

Let $\mathcal{X}=\mathbb{R}^l$, $\mathcal{Y}=\mathbb{R}^{n\times k}$, $\mathcal{Z}=\mathbb{R}^k$ and $\mathcal{H}=\mathbb{R}^{n}$. Let $\phi:\mathcal{X}\to\mathcal{Y}$ be a given collection of features and let $\{e_1,\dots,e_n\}$ be the standard basis for $\mathbb{R}^n$. Consider the continuous linear operator $L_{x,\phi}:\mathcal{H}\to\mathcal{Z}$ from Example \ref{ex:finiteexpbasis}, where $L_{x,\phi}(W) = \phi(x)^TW$. Then consider the $\ell_1$-regularization problem for feature selection given a set of observations $\mathcal{D} = \{(x_i,y_i):x_i\in\mathcal{X},y_i\in\mathcal{Z} ,i = 1,\dots,m\}$ given by,
\begin{equation}
\label{eq:l1prob}
\underset{W\in\mathcal{H}}{\text{min}} \quad \sum_{i=1}^{m} || y_i - L_{x_i,\phi}W ||_\mathcal{Z}^2 + \lambda || W ||_1^2
\end{equation}
where the $||W||_1=\left( \sum_{i=1}^{n}|W_{i}| \right)$. Given that the $\ell_1$ norm is orthomonotone with respect to $S_{proj}$, where we can write $S_{proj}(A)=\{\sum_{i}\lambda_i e_{i}:\lambda\in\mathbb{R},W\in A, W_{i}\neq 0 \}$ for any $A\subseteq\mathcal{H}$. Then from the Generalized Representer Theorem we know that a minimizer for \eqref{eq:l1prob} must exist in $S_{proj}(A)$, for $A=\mathcal{N}_{L_{x_1,\phi},\dots,L_{x_m,\phi}}^\perp = \{\sum_{i=1}^{m}L_{x_i,\phi}^*z_i:z_i\in\mathcal{Z}\}$.

From Example \ref{ex:finiteexpbasis}, we also know that $L_{x_i,\phi}^*z_i = \phi(x_i)z_i$. Thus $S_{proj}(A) = \{\sum_j \lambda_j e_j: \forall i \in\{1,\dots,m\}, \; \phi(x_i)^Te_j \neq 0, \lambda_j\in\mathbb{R} \}$. Substituting this form of the minimizer into \eqref{eq:l1prob}, we can then find the optimal $\lambda_j$s. The above problem is often used as a means for sparse feature selection in learning problems.

Note that while we showed the implications of having arbitrary Hilbert valued output spaces using the example of stochastic regression, we have not yet shown an implication of not having $r$-regular subspace valued maps. The examples from neural networks and stochastic regression were covered by $S_{\mathbb{R}}$ which is $1$-regular and the above $\ell_1$ problem was covered by $S_{proj}$ which for $n$-dimensional $\mathcal{H}$ is $n$-regular. To give an example of a subspace valued map that is not $r$-regular for any finite $r$ we must consider the $\ell_1$ regularization problem with $\mathcal{H}$ being a infinite dimensional Hilbert space for which the $\ell_1$ norm is well defined and an basis analogous to $\{e_1,\dots,e_n\}$ is available. 

For this purpose, let $\mathcal{X}=\mathbb{Z}$ be the set of integers and $\mathcal{Z}=\mathbb{R}$. Let $\mathcal{F}=2^\mathbb{Z}$ be a sigma algebra on $\mathcal{X}$ and $\mu$ be the counting measure on $(\mathcal{X},\mathcal{F})$ measurable space.
Let $\mathcal{H}$ be the space of $\ell_2(\mathcal{X},\mathcal{F},\mu)$ functions from $\mathcal{X}$ to $\mathbb{R}$ such that for any $f\in\mathcal{H}$, $||f||_2 = \sum_{i\in\mathbb{Z}} |f(i)|^2 <\infty$ and $\langle f,g \rangle_{\mathcal{H}} = \sum_{i\in\mathbb{Z}} f(i)g(i)$. Let $\langle z_1,z_2 \rangle_{\mathcal{Z}} = z_1z_2$ be the scalar product on $\mathcal{Z}=\mathbb{R}$.

Note that the $\ell_1$ norm is well defined for all $f\in\mathcal{H}$ as $||f||_1 = \sum_{i\in\mathbb{Z}} |f(i)|<\infty$.
Further a set of orthonormal basis for $\mathcal{H}$ can be written as  $\{(\delta_i:\mathcal{X}\to\mathcal{Z}): i\in\mathbb{Z}\}$ with $\delta_i$ defined as $\delta_i(j) = \begin{cases}
1 & \text{if } i=j \\
0 & \text{otherwise}
\end{cases}$. 
The above space of $\ell_2$ functions forms a complete Hilbert space as shown by \cite[Riesz-Fischer Theorem,][]{rudin1964principles}. 

Further note that the evaluation operator $L_x:\mathcal{H}\to\mathcal{Z}$ defined as $L_xf = f(x)$ for any $x\in\mathcal{X}$ is a bounded (implying continuous) linear operator on $\ell_2(\mathcal{X},\mathcal{F},\mu)$ with the adjoint $L_x^*$ given by $\delta_x(\cdot)$, since for all $z\in\mathbb{R}$, $\langle z,L_xf \rangle_{\mathcal{Z}}=zf(x) = \langle z\delta_x,f\rangle_{\mathcal{H}} =\langle L_x^*z,f\rangle_{\mathcal{H}}$.

Then for the problem,
\begin{equation}
\label{eq:l1prob2}
\underset{f\in\mathcal{H}}{\text{min}} \quad \sum_{i=1}^{m} || y_i - L_{x_i}f ||_\mathcal{Z}^2 + \lambda || f ||_1^2
\end{equation}
we have $\Omega(f) = || f ||_1^2$ orthomonotone with respect to the subspace valued map $S_{proj}(A) = \{\lambda \langle a,\delta_i \rangle_{\mathcal{H}}\delta_i: a\in A,\lambda\in\mathbb{R},i\in\mathbb{Z} \}$ (the proof for orthomonotonicity follows from similar arguments as presented in the proof for Lemma \ref{thm:l1ortho}). The $S_{proj}$ thus defined is not $r$-regular for any finite $r$. However by Theorem \ref{thm:NecSuff} we know the minimizer must be of the form $S_{proj}(\{\sum_{i=1}^{m}L_{x_i}^*z_i:z_i\in\mathbb{R}\}) = S_{proj}(\{\sum_{i=1}^{m}\delta_{x_i}(\cdot)z_i:z_i\in\mathbb{R}\})=\{\sum_{i=1}^{m}\delta_{x_i}(\cdot)z_i:z_i\in\mathbb{R}\}$. Thus \eqref{eq:l1prob2} provides an example of problems where a non $r$-regular subspace valued map is required and thus was not be covered by previous counterparts of the Generalized theorem.

\section{Conclusion}
We presented here an extension to existing work on generalized representer theorems by extending the result to apply to learning arbitrary Hilbert space-valued function spaces. Subspace valued maps with a super additive property were introduced and the property was shown to be necessary and sufficient for preserving a vector space structure. The assumption of ``r-regularity" was removed from the generalized theorem in order to allow more general subspace valued maps and its implications were shown for the $\ell_1$ regularization problem in function spaces. The formalism of linear operators and adjoints was introduced into the generalized representer theorem and new properties of subspace valued maps when composed with linear operators were established in order to achieve the said extension. The $\ell_1$ norm was shown to be orthomonotone with respect to a projection based subspace valued map that shows the sparsity inducing nature of the $\ell_1$ norm regularizers. Finally examples from kernel based neural networks, stochastic process learning and feature selection with $\ell_1$ norms were presented to show the application of the generalized theorem to these problems. 




\appendix
\section{Appendix}

\subsection{Subspace Valued Maps}
\begin{definition}{(Quasilinear map)\\}
	A map $S:\mathbb{F}\to\Span(\mathbb{F})$ is called \textbf{quasilinear} if 
	\[\forall A,B\in \mathbb{F}, \lambda_1,\lambda_2 \in \mathcal{K}, \qquad S(\lambda_1A+\lambda_2B)\subseteq \lambda_1S(A)+\lambda_2S(B) \]
\end{definition}
\begin{definition}{(Idempotent map)\\}
	A map $S:\mathbb{F}\to\Span(\mathbb{F})$ is called \textbf{idempotent} if \[\forall A\in \mathbb{F}, \qquad S(S(A))=S(A)\]
\end{definition}
\begin{definition}{($r$-regular maps)\\}
	For some $r\in\mathbb{N}$, we call a map $S:\mathbb{F}\to\mathbb{F}$, \textbf{$r$-regular} if
	\begin{enumerate}
		\item it is inclusive, quasilinear and idempotent
		\item for all $a\in U$, dimension of $S(a)$ is at most $r$
	\end{enumerate} 
\end{definition}
\SLprop*
\begin{proof}
	For any subspaces $A,B\subseteq \mathcal{H}$, $LA,LB\subseteq \mathcal{Z}$ are subspaces in $\mathcal{Z}$ ($\because$ $L$ is a continuous linear operator). Then $S_L(A+B)= L^*S(LA+LB)\supseteq L^*S(LA)+L^*S(LB)=S_L(A)+S_L(B)$ (by super additivity of $S$). Thus $S_L$ is super additive.
	Further if $S$ is closed then $S_L$ is also closed by default since continuous linear operators on Hilbert subspaces map closed sets to closed sets.
	Finally, note that $LA\subseteq S(LA)$ ($\because$ $S$ is inclusive). Then $L^*LA\subseteq L^*S(LA)$. However unless $A\subseteq L^*LA$ this does not imply inclusivity for $S_L$.
\end{proof}
\begin{lemma}
	\label{lemma:OrthS}
	Let $S_L$ be given by Proposition \ref{prop:pullbackS}. Then, $L(S_L(A)^\perp) \subseteq S(LA)^\perp$.
	\begin{proof}
		For all $z\in S(LA)$ and $g\in S_L(A)^\perp$, $\langle z, Lg \rangle_{\mathcal{Z}}=\langle L^*z, g \rangle_{\mathcal{H}}$. But by definition $L^*z \in S_L(A)$ ($\because$ $S_L(A)=L^*S(LA)$) and $g\in S_L(A)^\perp$, implying $\langle z, Lg \rangle_{\mathcal{Z}}=0$. Thus $L(S_L(A)^\perp)\subseteq S(LA)^\perp$.
	\end{proof}
\end{lemma}
\begin{lemma}{(Derivative operator: $S_L$ inclusive for $S_{proj}$)\\}
	\label{lem:Sproj_Dinclusive}
	Let $E_m=\{e_1,\dots,e_m\}$ be the standard orthonormal basis for $\mathbb{R}^m$ and $E_{mn}=\{e_1,\dots,e_{mn}\}$ be the standard orthonormal basis for $\mathbb{R}^{m\times n}$. Let $\mathcal{H}$ be a Hilbert space of $\mathbb{R}^m$-valued square integrable polynomial functions supported on $[-1,1]^n\subseteq \mathbb{R}^n$ with the Legendre polynomials, given as \{$p_{ij} e_i\in\mathcal{H}: p_{ij}(x)=c_j\partial^n_{x_i}[(x_i^2-1)^j],c_j = (j+0.5)^{\frac{1}{2}}(2^j j!)^{-1},j\in\mathbb{N},e_i\in E_m,x_i=\langle x,e_i\rangle_{\mathbb{R}^n}$\} as the orthonormal basis for $\mathcal{H}$. Let $\mathcal{Y}$ be the space of $\mathbb{R}^{m\times n}$-valued functions and $D:\mathcal{H}\to\mathcal{Y}$ be the derivative operator from Example \ref{ex:derOp}.
	Let $S_{proj}(A)=\{\sum_{j=0}^{\infty}\sum_{i=1}^{m}\lambda_{ij}\langle ae_i, p_{ij} e_i \rangle_\mathcal{H}e_ie_i^T:a\in A,\lambda_{ij}\in\mathbb{R} \}$ be an inclusive, closed, super additive subspace valued map on $\mathcal{Y}$. Then $S_L:\mathbb{F}_\mathcal{H}\to\Span(\mathbb{F}_\mathcal{H})$ defined as $S_L=D^*\circ S\circ D$ is inclusive, closed and super additive.
	\begin{proof}
		Closed and super additive $S_L$ is implied by Proposition \ref{prop:pullbackS}. To show inclusivity of $S_L$, let $\phi_i=\{j\in\mathbb{N}: \forall a\in A, \langle ae_i, p_{ij} e_i \rangle_\mathcal{H}\neq 0\}$ and note that $D^*S_{proj}(DA):=\{\sum_{i=1}^{m}\sum_{j\in \phi_i} \lambda_{ij}\partial_{x_i}^2p_{ij}e_i,\lambda_{ij}\in\mathbb{R}\}\supseteq A$ for any subspace $A=\{\sum_{i=1}^{m} \sum_{j\in \phi_i}\lambda_{ij}p_{ij}e_i\}$ (since monomials of all orders are still present in $D^*S_{proj}(DA)$). Thus $S_L$ is inclusive.
	\end{proof}
\end{lemma}
\begin{lemma}{(Derivative operator: $D$ preserves orthogonality w.r.t. $S_L$)\\}
	\label{lem:DpresOrtho}
	Given the space of $\mathcal{H}$ and $S_{proj}$ as defined in Lemma \ref{lem:Sproj_Dinclusive}, the derivative operator $D$ preserves orthogonality w.r.t. $S_L$.
	\begin{proof}
	$S_L$ is inclusive by Lemma \ref{lem:Sproj_Dinclusive}. Also $L(S_L(A)^\perp)\subseteq S(LA)^\perp$ by Lemma \ref{lemma:OrthS}. 

Further $L(S_L(A))= DD^*S_{proj}(DA):=\{\sum_{i=1}^{m}\sum_{j\in \phi_i} \lambda_{ij}\partial_{x_i}^3p_{ij}e_ie_i^T:\lambda_{ij}\in\mathbb{R}\}$ and $S_{proj}(DA)=\{\sum_{i=1}^{m}\sum_{j\in \phi_i} \lambda_{ij} \partial_{x_i} p_{ij}  e_ie_i^T : \lambda_{ij} \in \mathbb{R}\}$. Thus $DD^*S_{proj}(DA)\subseteq S_{proj}(DA)$ (since monomials of any order present in $DD^*S_{proj}(DA)$ are also contained in $S_{proj}(DA)$). Thus $L(S_L(A))\subseteq S(LA)$ and given the space of $\mathcal{H}$ and $S_{proj}$ as defined in Lemma \ref{lem:Sproj_Dinclusive}, the derivative operator $D$ preserves orthogonality w.r.t. $S_L$
	\end{proof}
\end{lemma}
Note on the other hand, using $S_{\mathbb{R}}$ instead of $S_{proj}$ does not preserve orthogonality with respect to the corresponding $S_L$.
\vskip 0.2in
\bibliography{RKHS,deepkernels}

\end{document}